\documentclass{article}

\usepackage{times}
\usepackage{graphicx} 
\usepackage{subfigure}

\usepackage{natbib}

\usepackage{algorithm}
\usepackage{algorithmic}
\usepackage[algo2e,linesnumbered,lined,boxed,vlined]{algorithm2e}

\usepackage{amsfonts}
\usepackage{amssymb}

\usepackage{hyperref}

\usepackage[accepted]{icml2016}

\usepackage{comment}
\usepackage{amsthm}
\usepackage{mathtools}
\usepackage{tabularx}
\usepackage{multirow}

\newtheorem{proposition}{Proposition}

\def\b{\ensuremath\boldsymbol}

\usepackage{lastpage}
\usepackage{fancyhdr}
\usepackage{url}
\icmltitlerunning{Stochastic Neighbor Embedding with Gaussian and Student-t Distributions: Tutorial and Survey}

\begin{document}

\AddToShipoutPictureBG*{%
  \AtPageUpperLeft{%
    \setlength\unitlength{1in}%
    \hspace*{\dimexpr0.5\paperwidth\relax}
    \makebox(0,-0.75)[c]{\normalsize {\color{black} To appear as a part of an upcoming academic book on dimensionality reduction and manifold learning.}}
    }}

\twocolumn[
\icmltitle{Stochastic Neighbor Embedding \\with Gaussian and Student-t Distributions: Tutorial and Survey}

\icmlauthor{Benyamin Ghojogh}{bghojogh@uwaterloo.ca}
\icmladdress{Department of Electrical and Computer Engineering, 
\\Machine Learning Laboratory, University of Waterloo, Waterloo, ON, Canada}
\icmlauthor{Ali Ghodsi}{ali.ghodsi@uwaterloo.ca}
\icmladdress{Department of Statistics and Actuarial Science \& David R. Cheriton School of Computer Science, 
\\Data Analytics Laboratory, University of Waterloo, Waterloo, ON, Canada}
\icmlauthor{Fakhri Karray}{karray@uwaterloo.ca}
\icmladdress{Department of Electrical and Computer Engineering, 
\\Centre for Pattern Analysis and Machine Intelligence, University of Waterloo, Waterloo, ON, Canada}
\icmlauthor{Mark Crowley}{mcrowley@uwaterloo.ca}
\icmladdress{Department of Electrical and Computer Engineering, 
\\Machine Learning Laboratory, University of Waterloo, Waterloo, ON, Canada}

\icmlkeywords{Tutorial, Locally Linear Embedding}

\vskip 0.3in
]

\begin{abstract}
Stochastic Neighbor Embedding (SNE) is a manifold learning and dimensionality reduction method with a probabilistic approach. In SNE, every point is consider to be the neighbor of all other points with some probability and this probability is tried to be preserved in the embedding space. SNE considers Gaussian distribution for the probability in both the input and embedding spaces. However, t-SNE uses the Student-t and Gaussian distributions in these spaces, respectively. In this tutorial and survey paper, we explain SNE, symmetric SNE, t-SNE (or Cauchy-SNE), and t-SNE with general degrees of freedom. We also cover the out-of-sample extension and acceleration for these methods.
\end{abstract}

\section{Introduction}

Stochastic Neighbor Embedding (SNE) \cite{hinton2003stochastic} is a manifold learning and dimensionality reduction method which can be used for feature extraction \cite{ghojogh2019feature}. It has a probabilistic approach. It fits the data in the embedding space locally hoping to preserve the global structure of data \cite{saul2003think}. 
The idea of SNE is to consider every point as neighbors of other points with some probability where the closer points are neighbors with higher probability. Therefore, rather than considering $k$ nearest neighbors in a binary manner (whether being neighbors or not), it considers neighbors in a stochastic way (for how probable it is to be neighbors). 
It tries to preserve the probability of neighborhoods in the low-dimensional embedding space. 
It is noteworthy that there exist some other similar probabilistic dimensionality reduction methods which make use of Gaussian distribution for neighborhood. Some examples are Neighborhood Component Analysis (NCA) \cite{goldberger2005neighbourhood}, deep NCA \cite{liu2020deep}, and Proxy-NCA \cite{movshovitz2017no}. 

SNE uses the Gaussian distribution for neighbors in both the input and embedding spaces. The Student-t distributed SNE, or so-called t-SNE \cite{maaten2008visualizing}, considers the Student-t and Gaussian distributions in the input and embedding spaces, respectively. The reason of using Student-t distribution in t-SNE is because of its heavier tails so it can include more information from the high-dimensional data. 
t-SNE is one of the state-of-the-art methods for data visualization; for example, it has been used for DNA and single-cell data visualization \cite{kobak2019art}.
In this paper, we explain SNE, symmetric SNE, t-SNE (or Cauchy-SNE), t-SNE with general degrees of freedom, their out-of-sample extensions, and their accelerations. We also show the results of simulations for visualization of embeddings. 

The goal of SNE is to embed the high-dimensional data $\{\b{x}_i\}_{i=1}^n$ into the lower dimensional data $\{\b{y}_i\}_{i=1}^n$ where $n$ is the number of data points. We denote the dimensionality of high- and low-dimensional spaces by $d$ and $p$, respectively, i.e. $\b{x}_i \in \mathbb{R}^d$ and $\b{y}_i \in \mathbb{R}^p$. We usually have $p \ll d$. For data visualization, we have $p \in \{2,3\}$. 

The remainder of this paper is organized as follows. In Sections \ref{section_SNE} and \ref{section_symmetric_SNE}, we explain SNE and symmetric SNE, respectively. Section \ref{section_tSNE} introduces the crowding problem and the t-SNE or Cauchy-SNE method. The out-of-sample embedding and acceleration of these methods are introduced in Sections \ref{section_outOfSample} and \ref{section_acceleration}, respectively. 
Recent improvements of t-SNE are briefly mentioned in Section \ref{section_improvement_of_tSNE}.
Finally, Section \ref{section_conclusion} concludes the paper. 

\section{Stochastic Neighbor Embedding (SNE)}\label{section_SNE}

In \textit{SNE} \cite{hinton2003stochastic}, we consider a Gaussian probability around every point $\b{x}_i$ where the point $\b{x}_i$ is on the mean and the distribution is for probability of accepting any other point as the neighbor of $\b{x}_i$; the farther points are neighbors with less probability. 
Hence, the variable is distance, denoted by $d \in \mathbb{R}$, and the Gaussian probability is:
\begin{align}
f(d) = \frac{1}{\sqrt{2 \pi \sigma^2}} \exp(-\frac{d^2}{2 \sigma^2}),
\end{align}
where the mean of distribution is assumed to be zero. 
The fixed multiplier $\frac{1}{\sqrt{2 \pi \sigma^2}}$ can be dropped; however, $\exp(-d^2 / 2 \sigma^2)$ does not add (integrate) to one and thus it is not a probability density function. In order to tackle this problem, we can do a trick and divide $\exp(-d^2 / 2 \sigma^2)$ by the summation of all possible values of $\exp(-d^2 / 2 \sigma^2)$ to have a \textit{softmax} function. Therefore, the probability that the point $\b{x}_i \in \mathbb{R}^d$ takes $\b{x}_j \in \mathbb{R}^d$ as its neighbor is:
\begin{align}\label{equation_SNE_p}
\mathbb{R} \ni p_{ij} := \frac{\exp(-d_{ij}^2)}{\sum_{k \neq i}\exp(-d_{ik}^2)},
\end{align}
where:
\begin{align}
\mathbb{R} \ni d_{ij}^2 := \frac{||\b{x}_i - \b{x}_j||_2^2}{2\sigma_i^2}.
\end{align}
Note that this trick is also used for $q_{ij}$ in SNE and also for $p_{ij}$ and $q_{ij}$ in t-SNE (and its variants) as we will see later. 

It is noteworthy that the mentioned trick is also used in other methods such as Continuous Bag-of-Word (CBOW) model of Word2Vec \cite{mikolov2013efficient,rong2014word2vec}, Euclidean Embedding \cite{globerson2007euclidean}, and Parametric Embedding \cite{iwata2005parametric}. In this trick, the summation in the denominator can get very time-consuming especially when the dataset (or corpus for Word2Vec) is large. This plus the slow pace of gradient descent \cite{boyd2004convex} are the reasons that SNE, t-SNE, and Word2Vec are very slow and even infeasible for large datasets. The Word2Vec tackled the problem of the slow pace by introducing Negative Sampling Skip-Gram model \cite{mikolov2013distributed,goldberg2014word2vec} which uses logistic function similar to the approach of logistic regression \cite{kleinbaum2002logistic}. In logistic function, we deal with inner product (similarity) of data points rather than distance of data points and there is no summation in the denominator.
The Negative Sampling Skip-Gram model also uses Newton's method, which is much faster than gradient descent \cite{boyd2004convex}, similar to logistic regression.

The $\sigma_i^2$ is the variance which we consider for the Gaussian distribution used for the $\b{x}_i$.
It can be set to a fixed number or determined by a binary search to make the entropy of distribution some specific value \cite{hinton2003stochastic}. Note that according to the distribution of data in the input space, the best value for the variance of Gaussian distributions can be found. 

In the low-dimensional embedding space, we again consider a Gaussian probability distribution for the point $\b{y}_i \in \mathbb{R}^p$ to take $\b{y}_j \in \mathbb{R}^p$ as its neighbor:
\begin{align}\label{equation_SNE_q}
\mathbb{R} \ni q_{ij} := \frac{\exp(-z_{ij}^2)}{\sum_{k \neq i}\exp(-z_{ik}^2)},
\end{align}
where:
\begin{align}\label{equation_SNE_z_squared}
\mathbb{R} \ni z_{ij}^2 := ||\b{y}_i - \b{y}_j||_2^2.
\end{align}
It is noteworthy that the variance of distribution is not used (or is set to $\sigma_i^2 = 0.5$ to cancel $2$ in the denominator) because the variance of distribution in the embedding space is the choice of algorithm.

We want the probability distributions in both the input and embedded spaces to be as similar as possible; therefore, the cost function to be minimized can be summation of the Kullback-Leibler (KL) divergences \cite{kullback1997information} over the $n$ points:
\begin{align}
\mathbb{R} \ni c_1 := \sum_{i=1}^n \text{KL}(P_i||Q_i) = \sum_{i=1}^n \sum_{j=1,j \neq i}^{n} p_{ij} \log (\frac{p_{ij}}{q_{ij}}),
\end{align}
where $p_{ij}$ and $q_{ij}$ are the Eqs. (\ref{equation_SNE_p}) and (\ref{equation_SNE_q}).
Note that divergences other than the KL divergence can be used for the SNE optimization; e.g., see \cite{im2018stochastic}.
\begin{proposition}
The gradient of $c_1$ with respect to $\b{y}_i$ is:
\begin{align}
\mathbb{R}^p \ni \frac{\partial c_1}{\partial \b{y}_i} = 2 \sum_{j=1}^n (p_{ij} - q_{ij} + p_{ji} - q_{ji}) (\b{y}_i - \b{y}_j),
\end{align}
where $p_{ij}$ and $q_{ij}$ are the Eqs. (\ref{equation_SNE_p}) and (\ref{equation_SNE_q}), and $p_{ii} = q_{ii} = 0$.
\end{proposition}
\begin{proof}
Proof is inspired by \cite{maaten2008visualizing}. Let:
\begin{align}\label{equation_SNE_r}
\mathbb{R} \ni r_{ij} := z_{ij}^2 = ||\b{y}_i - \b{y}_j||_2^2.
\end{align}
By changing $\b{y}_i$, we only have change impact in $z_{ij}$ and $z_{ji}$ (or $r_{ij}$ and $r_{ji}$) for all $j$'s.
According to chain rule, we have:
\begin{align*}
\mathbb{R}^p \ni \frac{\partial c_1}{\partial \b{y}_i} = \sum_j \big(\frac{\partial c_1}{\partial r_{ij}} \frac{\partial r_{ij}}{\partial \b{y}_i} + \frac{\partial c_1}{\partial r_{ji}} \frac{\partial r_{ji}}{\partial \b{y}_i}\big).
\end{align*}
According to Eq. (\ref{equation_SNE_r}), we have:
\begin{align*}
& r_{ij} = ||\b{y}_i - \b{y}_j||_2^2 \implies \frac{\partial r_{ij}}{\partial \b{y}_i} = 2(\b{y}_i - \b{y}_j), \\
& r_{ji} = ||\b{y}_j - \b{y}_i||_2^2 = ||\b{y}_i - \b{y}_j||_2^2 \implies \frac{\partial r_{ji}}{\partial \b{y}_i} = 2(\b{y}_i - \b{y}_j).
\end{align*}
Therefore:
\begin{align}\label{equation_SNE_deriv_c1_y}
\therefore ~~~~ \frac{\partial c_1}{\partial \b{y}_i} = 2 \sum_j \big(\frac{\partial c_1}{\partial r_{ij}} + \frac{\partial c_1}{\partial r_{ji}}\big)(\b{y}_i - \b{y}_j).
\end{align}
The cost function can be re-written as:
\begin{align*}
c_1 &= \sum_{k} \sum_{l\neq k} p_{kl} \log (\frac{p_{kl}}{q_{kl}}) = \sum_{k \neq l} p_{kl} \log (\frac{p_{kl}}{q_{kl}}) \\
&= \sum_{k \neq l} \big(p_{kl} \log (p_{kl}) - p_{kl} \log (q_{kl}) \big),
\end{align*}
whose first term is a constant with respect to $q_{kl}$ and thus to $r_{kl}$. We have:
\begin{align*}
\mathbb{R} \ni \frac{\partial c_1}{\partial r_{ij}} = - \sum_{k \neq l} p_{kl} \frac{\partial (\log (q_{kl}))}{\partial r_{ij}}.
\end{align*}
According to Eq. (\ref{equation_SNE_q}), the $q_{kl}$ is:
\begin{align*}
q_{kl} := \frac{\exp(-z_{kl}^2)}{\sum_{k \neq f}\exp(-z_{kf}^2)} = \frac{\exp(-r_{kl})}{\sum_{k \neq f}\exp(-r_{kf})}.
\end{align*}
We take the denominator of $q_{kl}$ as:
\begin{align}\label{equation_SNE_beta}
\beta := \sum_{k \neq f} \exp(- z_{kf}^2) = \sum_{k \neq f} \exp(- r_{kf}).
\end{align}
We have $\log (q_{kl}) = \log (q_{kl}) + \log \beta - \log \beta = \log (q_{kl} \beta) - \log \beta$. Therefore:
\begin{align*}
\therefore ~~~ \frac{\partial c_1}{\partial r_{ij}} &= - \sum_{k \neq l} p_{kl} \frac{\partial \big(\log (q_{kl} \beta) - \log \beta\big)}{\partial r_{ij}} \\
&= - \sum_{k \neq l} p_{kl} \bigg[\frac{\partial \big(\log (q_{kl} \beta)\big)}{\partial r_{ij}} - \frac{\partial \big(\log \beta\big)}{\partial r_{ij}}\bigg] \\
&= - \sum_{k \neq l} p_{kl} \bigg[\frac{1}{q_{kl} \beta}\frac{\partial \big( q_{kl} \beta\big)}{\partial r_{ij}} - \frac{1}{\beta}\frac{\partial \beta}{\partial r_{ij}}\bigg].
\end{align*}
The $q_{kl} \beta$ is:
\begin{align*}
q_{kl} \beta &= \frac{\exp(-r_{kl})}{\sum_{f \neq k}\exp(-r_{kf})} \times \sum_{k \neq f} \exp(- r_{kf}) \\
&= \exp(-r_{kl}).
\end{align*}
Therefore, we have:
\begin{align*}
\therefore ~~~ \frac{\partial c_1}{\partial r_{ij}} &= - \sum_{k \neq l} p_{kl} \bigg[\frac{1}{q_{kl} \beta}\frac{\partial \big( \exp(-r_{kl}) \big)}{\partial r_{ij}} - \frac{1}{\beta}\frac{\partial \beta}{\partial r_{ij}}\bigg].
\end{align*}
The $\partial \big( \exp(-r_{kl}) \big)/\partial r_{ij}$ is non-zero for only $k=i$ and $l=j$; therefore:
\begin{align*}
\frac{\partial \big( \exp(-r_{ij}) \big)}{\partial r_{ij}} &= - \exp(-r_{ij}), \\
\frac{\partial \beta}{\partial r_{ij}} &= \frac{\partial \sum_{k \neq f} \exp(- r_{kf})}{\partial r_{ij}} = \frac{\partial \exp(- r_{ij})}{\partial r_{ij}} \\
&= - \exp(- r_{ij}).
\end{align*}
Therefore:
\begin{alignat*}{2}
\therefore ~~~ \frac{\partial c_1}{\partial r_{ij}} &= &&- \bigg( p_{ij} \Big[\frac{-1}{q_{ij} \beta} \exp(-r_{ij})\Big] + 0 + \dots + 0 \bigg) \\
& &&- \sum_{k \neq l} p_{k l} \Big[\frac{1}{\beta} \exp(- r_{ij}) \Big].
\end{alignat*}
We have $\sum_{k \neq l} p_{k l} = 1$ because summation of all possible probabilities is one. Thus:
\begin{align}
\frac{\partial c_1}{\partial r_{ij}} &= -  p_{ij} \Big[\frac{-1}{q_{ij} \beta} \exp(-r_{ij})\Big] - \Big[\frac{1}{\beta} \exp(- r_{ij}) \Big] \nonumber \\
&= \underbrace{\frac{\exp(- r_{ij})}{\beta}}_{=q_{ij}} \Big[\frac{p_{ij}}{q_{ij}} - 1\Big] = p_{ij}  - q_{ij}. \label{equation_SNE_derivative_r_ij}
\end{align}
Similarly, we have:
\begin{align}\label{equation_SNE_derivative_r_ji}
\frac{\partial c_1}{\partial r_{ji}} = p_{ji}  - q_{ji}.
\end{align}
Substituting the obtained derivatives in Eq. (\ref{equation_SNE_deriv_c1_y}) gives us:
\begin{align*}
\frac{\partial c_1}{\partial \b{y}_i} = 2 \sum_j (p_{ij} - q_{ij} + p_{ji} - q_{ji})(\b{y}_i - \b{y}_j),
\end{align*}
which is the gradient mentioned in the proposition. Q.E.D.
\end{proof}

The update of the embedded point $\b{y}_i$ is done by gradient descent. Every iteration is:
\begin{equation}\label{equation_SNE_gradient_descent}
\begin{aligned}
& \Delta \b{y}_i^{(t)} := - \eta\, \frac{\partial c_1}{\partial \b{y}_i} + \alpha(t)\, \Delta \b{y}_i^{(t-1)}, \\
& \b{y}_i^{(t)} := \b{y}_i^{(t-1)} + \Delta \b{y}_i^{(t)},
\end{aligned}
\end{equation}
where momentum is used for better convergence \cite{qian1999momentum}. 
The $\alpha(t)$ is the momentum. It can be smaller for initial iterations and larger for further iterations. For example, we can have \cite{maaten2008visualizing}:
\begin{align}\label{equation_alpha_momentum}
\alpha(t) := \left\{
    \begin{array}{ll}
        0.5 & t < 250, \\
        0.8 & t \geq 250.
    \end{array}
\right.
\end{align}
In the original paper of SNE \cite{hinton2003stochastic}, the momentum term is not mentioned but it is suggested in \cite{maaten2008visualizing}.

The $\eta$ is the learning rate which can be a small positive constant (e.g., $\eta=0.1$) or can be updated adaptively according to \cite{jacobs1988increased}.

Moreover, in both \cite{hinton2003stochastic} and \cite{maaten2008visualizing}, it is mentioned that in SNE we should add some Gaussian noise (random jitter) to the solution of the first iterations before going to the next iterations. It helps avoiding the local optimum solutions.

\section{Symmetric Stochastic Neighbor Embedding}\label{section_symmetric_SNE}

In \textit{symmetric SNE} \cite{maaten2008visualizing}, we consider a Gaussian probability around every point $\b{x}_i$.
The probability that the point $\b{x}_i \in \mathbb{R}^d$ takes $\b{x}_j \in \mathbb{R}^d$ as its neighbor is:
\begin{align}\label{equation_sym_SNE_p_notUsed}
\mathbb{R} \ni p_{ij} := \frac{\exp(-d_{ij}^2)}{\sum_{k \neq l}\exp(-d_{kl}^2)},
\end{align}
where:
\begin{align}
\mathbb{R} \ni d_{ij}^2 := \frac{||\b{x}_i - \b{x}_j||_2^2}{2\sigma_i^2}.
\end{align}
Note that the denominator of Eq. (\ref{equation_sym_SNE_p_notUsed}) for all points is fixed and thus it is symmetric for $i$ and $j$. Compare this with Eq. (\ref{equation_SNE_p}) which is not symmetric.

The $\sigma_i^2$ is the variance which we consider for the Gaussian distribution used for the $\b{x}_i$.
It can be set to a fixed number or determined by a binary search to make the entropy of distribution some specific value \cite{hinton2003stochastic}.

The Eq. (\ref{equation_sym_SNE_p_notUsed}) has a problem with outliers. If the point $\b{x}_i$ is an outlier, its $p_{ij}$ will be extremely small because the denominator is fixed for every point and numerator will be small for the outlier. However, If we use Eq. (\ref{equation_SNE_p}) for $p_{ij}$, the denominator for all the points is not the same and therefore, the denominator for an outlier will also be small waving out the problem of small numerator. For this mentioned problem, we do not use Eq. (\ref{equation_sym_SNE_p_notUsed}) and rather we use:
\begin{align}\label{equation_sym_SNE_p}
\mathbb{R} \ni p_{ij} := \frac{p_{i|j} + p_{j|i}}{2n},
\end{align}
where:
\begin{align}\label{equation_sym_SNE_p_conditional}
\mathbb{R} \ni p_{j|i} := \frac{\exp(-d_{ij}^2)}{\sum_{k \neq i}\exp(-d_{ik}^2)},
\end{align}
is the probability that $\b{x}_i \in \mathbb{R}^d$ takes $\b{x}_j \in \mathbb{R}^d$ as its neighbor.

In the low-dimensional embedding space, we consider a Gaussian probability distribution for the point $\b{y}_i \in \mathbb{R}^p$ to take $\b{y}_j \in \mathbb{R}^p$ as its neighbor and we make it symmetric (fixed denominator for all points):
\begin{align}\label{equation_sym_SNE_q}
\mathbb{R} \ni q_{ij} := \frac{\exp(-z_{ij}^2)}{\sum_{k \neq l}\exp(-z_{kl}^2)},
\end{align}
where:
\begin{align}\label{equation_sym_SNE_z_squared}
\mathbb{R} \ni z_{ij}^2 := ||\b{y}_i - \b{y}_j||_2^2.
\end{align}
Note that the Eq. (\ref{equation_sym_SNE_q}) does not have the problem of outliers as in Eq. (\ref{equation_sym_SNE_p_notUsed}) because even for an outlier, the embedded points are initialized close together and not far.

We want the probability distributions in both the input and embedded spaces to be as similar as possible; therefore, the cost function to be minimized can be summation of the Kullback-Leibler (KL) divergences \cite{kullback1997information} over the $n$ points:
\begin{align}
\mathbb{R} \ni c_2 := \sum_{i=1}^n \text{KL}(P_i||Q_i) = \sum_{i=1}^n \sum_{j=1,j \neq i}^n p_{ij} \log (\frac{p_{ij}}{q_{ij}}),
\end{align}
where $p_{ij}$ and $q_{ij}$ are the Eqs. (\ref{equation_sym_SNE_p}) and (\ref{equation_sym_SNE_q}).
\begin{proposition}
The gradient of $c_2$ with respect to $\b{y}_i$ is:
\begin{align}
\mathbb{R}^p \ni \frac{\partial c_2}{\partial \b{y}_i} = 4 \sum_{j=1}^n (p_{ij} - q_{ij}) (\b{y}_i - \b{y}_j),
\end{align}
where $p_{ij}$ and $q_{ij}$ are the Eqs. (\ref{equation_sym_SNE_p}) and (\ref{equation_sym_SNE_q}), and $p_{ii} = q_{ii} = 0$.
\end{proposition}
\begin{proof}
Proof is inspired by \cite{maaten2008visualizing}. 

Similar to Eq. (\ref{equation_SNE_deriv_c1_y}), we have:
\begin{align}\label{equation_sym_SNE_deriv_c1_y}
\frac{\partial c_2}{\partial \b{y}_i} = 2 \sum_j \big(\frac{\partial c_2}{\partial r_{ij}} + \frac{\partial c_2}{\partial r_{ji}}\big)(\b{y}_i - \b{y}_j).
\end{align}
Similar to the derivation of Eqs. (\ref{equation_SNE_derivative_r_ij}) and (\ref{equation_SNE_derivative_r_ji}), we can derive:
\begin{align}
&\frac{\partial c_2}{\partial r_{ij}} = p_{ij}  - q_{ij}, \text{   and} \label{equation_sym_SNE_derivative_r_ij} \\
&\frac{\partial c_2}{\partial r_{ji}} = p_{ji}  - q_{ji}, \nonumber
\end{align}
respectively. 
In the symmetric SNE, we have:
\begin{align}\label{equation_sym_SNE_derivative_r_ji}
\frac{\partial c_2}{\partial r_{ji}} = p_{ji}  - q_{ji} \overset{(a)}{=} p_{ij}  - q_{ij},
\end{align}
where $(a)$ is because in symmetric SNE, the $p_{ij}$ and $q_{ij}$ are symmetric for $i$ and $j$ according to Eqs. (\ref{equation_sym_SNE_p}) and (\ref{equation_sym_SNE_q}).

Substituting Eqs. (\ref{equation_sym_SNE_derivative_r_ij}) and (\ref{equation_sym_SNE_derivative_r_ji}) in Eq. (\ref{equation_sym_SNE_deriv_c1_y}) gives us:
\begin{align*}
\frac{\partial c_2}{\partial \b{y}_i} = 4 \sum_j (p_{ij} - q_{ij})(\b{y}_i - \b{y}_j),
\end{align*}
which is the gradient mentioned in the proposition. Q.E.D.
\end{proof}

The update of the embedded point $\b{y}_i$ is done by gradient descent whose every iteration is as Eq. (\ref{equation_SNE_gradient_descent}) where $c_1$ is replaced by $c_2$.
Note that the momentum term can be omitted in the symmetric SNE.
Like in SNE, in symmetric SNE, we should add some Gaussian noise (random jitter) to the solution of the first iterations before going to the next iterations. It helps avoiding the local optimum solutions.



\section{t-distributed Stochastic Neighbor Embedding (t-SNE)}\label{section_tSNE}

\subsection{The Crowding Problem}\label{section_crowding_problem}

In SNE \cite{hinton2003stochastic}, we are considering Gaussian distribution for both input and embedded spaces. That is okay for the input space because it already has a high dimensionality. However, when we embed the high-dimensional data into a low-dimensional space, it is very hard to fit the information of all the points in the same neighborhood area. For better clarification, suppose the dimensionality is like the size of a room. In high dimensionality, we have a large hall including a huge crowd of people. Now, we want to fit all the people into a small room; of course, we cannot! This problem is referred to as the \textit{crowding problem}. 


The main idea of \textit{t-SNE} \cite{maaten2008visualizing} is addressing the crowding problem which exists in SNE \cite{hinton2003stochastic}. 
In the example of fitting people in a room, t-SNE enlarges the room to solve the crowding problem. 
Therefore, in the formulation of t-SNE, we use Student-t distribution \cite{student1908probable} rather than Gaussian distribution for the low-dimensional embedded space. This is because the Student-t distribution has heavier tails than Gaussian distribution, which is like a larger room, and can fit the information of high dimensional data in the low dimensional embedding space. 

As we will see later, the $q_{ij}$ in t-SNE is:
\begin{align*}
q_{ij} = \frac{(1 + z_{ij}^2)^{-1}}{\sum_{k \neq l}(1 + z_{kl}^2)^{-1}},
\end{align*}
which is based on the standard Cauchy distribution:
\begin{align}
f(z) = \frac{1}{\pi(1 + z^2)},
\end{align}
where $\pi$ is canceled from the numerator and the normalizing denominator in $q_{ij}$ (see the explanations of this trick in Section \ref{section_SNE}).

If the Student-t distribution \cite{student1908probable} with the general degrees of freedom $\delta$ is used, we would have:
\begin{align}\label{equation_t_SNE_t_distribution}
f(z) = \frac{\Gamma(\frac{\delta+1}{2})}{\sqrt{\delta \times \pi}~~\Gamma(\frac{\delta}{2})} (1 + \frac{z^2}{\delta})^{-\frac{\delta+1}{2}},
\end{align}
where $\Gamma$ is the gamma function.
Cancelling out the scaling factors from the numerator and denominator, we would have \cite{van2009learning}:
\begin{align}\label{equation_t_SNE_q_generalDegrees}
q_{ij} = \frac{(1 + z_{ij}^2/\delta)^{-(\delta+1)/2}}{\sum_{k \neq l}(1 + z_{kl}^2/\delta)^{-(\delta+1)/2}}.
\end{align}
However, as the first degree of freedom has the heaviest tails amongst different degrees of freedom, it is the most suitable for the crowding problem; hence, we use the first degree of freedom which is the Cauchy distribution.
Note that the t-SNE algorithm, which uses the Cauchy distribution, may also be  called the Cauchy-SNE. 
Later, t-SNE with general degrees of freedom was proposed \cite{van2009learning}, which we explain in Section \ref{section_tSNE_general_degrees}. 

\subsection{t-SNE Formulation}

In t-SNE \cite{maaten2008visualizing}, we consider a Gaussian probability around every point $\b{x}_i$ in the input space because the crowding problem does not exist in the high dimensional data.
The probability that the point $\b{x}_i \in \mathbb{R}^d$ takes $\b{x}_j \in \mathbb{R}^d$ as its neighbor is:
\begin{align}\label{equation_t_SNE_p_conditional}
\mathbb{R} \ni p_{j|i} := \frac{\exp(-d_{ij}^2)}{\sum_{k \neq i}\exp(-d_{ik}^2)},
\end{align}
where:
\begin{align}
\mathbb{R} \ni d_{ij}^2 := \frac{||\b{x}_i - \b{x}_j||_2^2}{2\sigma_i^2}.
\end{align}
Note that Eq. (\ref{equation_t_SNE_p_conditional}) is not symmetric for $i$ and $j$ because of the denominator. 
We take the symmetric $p_{ij}$ as the scaled average of $p_{i|j}$ and $p_{j|i}$:
\begin{align}\label{equation_t_SNE_p}
\mathbb{R} \ni p_{ij} := \frac{p_{i|j} + p_{j|i}}{2n}.
\end{align}
In the low-dimensional embedding space, we consider a Student's $t$-distribution with one degree of freedom (Cauchy distribution) for the point $\b{y}_i \in \mathbb{R}^p$ to take $\b{y}_j \in \mathbb{R}^p$ as its neighbor:
\begin{align}\label{equation_t_SNE_q}
\mathbb{R} \ni q_{ij} := \frac{(1 + z_{ij}^2)^{-1}}{\sum_{k \neq l}(1 + z_{kl}^2)^{-1}},
\end{align}
where:
\begin{align}\label{equation_t_SNE_z_squared}
\mathbb{R} \ni z_{ij}^2 := ||\b{y}_i - \b{y}_j||_2^2.
\end{align}
We want the probability distributions in both the input and embedded spaces to be as similar as possible; therefore, the cost function to be minimized can be summation of the Kullback-Leibler (KL) divergences \cite{kullback1997information} over the $n$ points:
\begin{align}\label{equation_t_SNE_cost}
\mathbb{R} \ni c_3 := \sum_{i=1}^n \text{KL}(P_i||Q_i) = \sum_{i=1}^n \sum_{j=1,j \neq i}^n p_{ij} \log (\frac{p_{ij}}{q_{ij}}),
\end{align}
where $p_{ij}$ and $q_{ij}$ are the Eqs. (\ref{equation_t_SNE_p}) and (\ref{equation_t_SNE_q}).
\begin{proposition}
The gradient of $c_3$ with respect to $\b{y}_i$ is:
\begin{align}\label{equation_tSNE_gradient}
\frac{\partial c_3}{\partial \b{y}_i} = 4 \sum_{j=1}^n (p_{ij} - q_{ij})(1 + ||\b{y}_i - \b{y}_j||_2^2)^{-1}(\b{y}_i - \b{y}_j),
\end{align}
where $p_{ij}$ and $q_{ij}$ are the Eqs. (\ref{equation_t_SNE_p}) and (\ref{equation_t_SNE_q}), and $p_{ii} = q_{ii} = 0$.
\end{proposition}

\begin{proof}
Proof is according to \cite{maaten2008visualizing}. Let:
\begin{align}\label{equation_t_SNE_r}
\mathbb{R} \ni r_{ij} := z_{ij}^2 = ||\b{y}_i - \b{y}_j||_2^2.
\end{align}
By changing $\b{y}_i$, we only have change impact in $z_{ij}$ and $z_{ji}$ for all $j$'s.
According to chain rule, we have:
\begin{align*}
\mathbb{R}^p \ni \frac{\partial c_3}{\partial \b{y}_i} = \sum_j \big(\frac{\partial c_3}{\partial r_{ij}} \frac{\partial r_{ij}}{\partial \b{y}_i} + \frac{\partial c_3}{\partial r_{ji}} \frac{\partial r_{ji}}{\partial \b{y}_i}\big).
\end{align*}
According to Eq. (\ref{equation_t_SNE_r}), we have:
\begin{align*}
& r_{ij} = ||\b{y}_i - \b{y}_j||_2^2 \implies \frac{\partial r_{ij}}{\partial \b{y}_i} = 2(\b{y}_i - \b{y}_j), \\
& r_{ji} = ||\b{y}_j - \b{y}_i||_2^2 = ||\b{y}_i - \b{y}_j||_2^2 \implies \frac{\partial r_{ji}}{\partial \b{y}_i} = 2(\b{y}_i - \b{y}_j).
\end{align*}
Therefore:
\begin{align}\label{equation_t_SNE_deriv_c1_y}
\therefore ~~~~ \frac{\partial c_3}{\partial \b{y}_i} = 2 \sum_j \big(\frac{\partial c_3}{\partial r_{ij}} + \frac{\partial c_3}{\partial r_{ji}}\big)(\b{y}_i - \b{y}_j).
\end{align}
The cost function can be re-written as:
\begin{align*}
c_3 &= \sum_{k} \sum_{l\neq k} p_{kl} \log (\frac{p_{kl}}{q_{kl}}) = \sum_{k \neq l} p_{kl} \log (\frac{p_{kl}}{q_{kl}}) \\
&= \sum_{k \neq l} \big(p_{kl} \log (p_{kl}) - p_{kl} \log (q_{kl}) \big),
\end{align*}
whose first term is a constant with respect to $q_{kl}$ and thus to $r_{kl}$. We have:
\begin{align*}
\mathbb{R} \ni \frac{\partial c_3}{\partial r_{ij}} = - \sum_{k \neq l} p_{kl} \frac{\partial (\log (q_{kl}))}{\partial r_{ij}}.
\end{align*}
According to Eq. (\ref{equation_t_SNE_q}), the $q_{kl}$ is:
\begin{align*}
q_{kl} := \frac{(1 + z_{kl}^2)^{-1}}{\sum_{m \neq f}(1 + z_{mf}^2)^{-1}}, = \frac{(1 + r_{kl})^{-1}}{\sum_{m \neq f}(1 + r_{mf})^{-1}},
\end{align*}
We take the denominator of $q_{kl}$ as:
\begin{align}\label{equation_t_SNE_beta}
\beta := \sum_{m \neq f}(1 + z_{mf}^2)^{-1} = \sum_{m \neq f}(1 + r_{mf})^{-1}.
\end{align}
We have $\log (q_{kl}) = \log (q_{kl}) + \log \beta - \log \beta = \log (q_{kl} \beta) - \log \beta$. Therefore:
\begin{align*}
\therefore ~~~ \frac{\partial c_3}{\partial r_{ij}} &= - \sum_{k \neq l} p_{kl} \frac{\partial \big(\log (q_{kl} \beta) - \log \beta\big)}{\partial r_{ij}} \\
&= - \sum_{k \neq l} p_{kl} \bigg[\frac{\partial \big(\log (q_{kl} \beta)\big)}{\partial r_{ij}} - \frac{\partial \big(\log \beta\big)}{\partial r_{ij}}\bigg] \\
&= - \sum_{k \neq l} p_{kl} \bigg[\frac{1}{q_{kl} \beta}\frac{\partial \big( q_{kl} \beta\big)}{\partial r_{ij}} - \frac{1}{\beta}\frac{\partial \beta}{\partial r_{ij}}\bigg].
\end{align*}
The $q_{kl} \beta$ is:
\begin{align*}
q_{kl} \beta &= \frac{(1 + r_{kl})^{-1}}{\sum_{m \neq f}(1 + r_{mf})^{-1}} \times \sum_{m \neq f}(1 + r_{mf})^{-1} \\
& = (1 + r_{kl})^{-1}.
\end{align*}
Therefore, we have:
\begin{align*}
\therefore ~~~ \frac{\partial c_3}{\partial r_{ij}} &= - \sum_{k \neq l} p_{kl} \bigg[\frac{1}{q_{kl} \beta}\frac{\partial \big( (1 + r_{kl})^{-1} \big)}{\partial r_{ij}} - \frac{1}{\beta}\frac{\partial \beta}{\partial r_{ij}}\bigg].
\end{align*}
The $\partial \big( (1 + r_{kl})^{-1} \big)/\partial r_{ij}$ is non-zero for only $k=i$ and $l=j$; therefore:
\begin{align*}
\frac{\partial \big( (1 + r_{ij})^{-1} \big)}{\partial r_{ij}} &= - (1 + r_{ij})^{-2}, \\
\frac{\partial \beta}{\partial r_{ij}} &= \frac{\partial \sum_{m \neq f}(1 + r_{mf})^{-1}}{\partial r_{ij}} = \frac{\partial (1 + r_{ij})^{-1}}{\partial r_{ij}} \\
&= - (1 + r_{ij})^{-2}.
\end{align*}
Therefore:
\begin{alignat*}{2}
\therefore ~~~ \frac{\partial c_3}{\partial r_{ij}} &= &&- \bigg( p_{ij} \Big[\frac{-1}{q_{ij} \beta} (1 + r_{ij})^{-2}\Big] + 0 + \dots + 0 \bigg) \\
& &&- \sum_{k \neq l} p_{k l} \Big[\frac{1}{\beta} (1 + r_{ij})^{-2} \Big].
\end{alignat*}
We have $\sum_{k \neq l} p_{kl} = 1$ because summation of all possible probabilities is one. Thus:
\begin{align*}
\frac{\partial c_3}{\partial r_{ij}} &= -  p_{ij} \Big[\frac{-1}{q_{ij} \beta} (1 + r_{ij})^{-2}\Big] - \Big[\frac{1}{\beta} (1 + r_{ij})^{-2} \Big] \\
&= (1 + r_{ij})^{-1} \underbrace{\frac{(1 + r_{ij})^{-1}}{\beta}}_{=q_{ij}} \Big[\frac{p_{ij}}{q_{ij}} - 1\Big] \\
&= (1 + r_{ij})^{-1} (p_{ij}  - q_{ij}).
\end{align*}
Similarly, we have:
\begin{align*}
\frac{\partial c_3}{\partial r_{ji}} = (1 + r_{ji})^{-1} (p_{ji}  - q_{ji}) \overset{(a)}{=} (1 + r_{ij})^{-1} (p_{ij}  - q_{ij}),
\end{align*}
where $(a)$ is because in t-SNE, the $p_{ij}$, $q_{ij}$, and $r_{ij}$ are symmetric for $i$ and $j$ according to Eqs. (\ref{equation_t_SNE_p}), (\ref{equation_t_SNE_q}), and (\ref{equation_t_SNE_r}).

Substituting the obtained derivatives in Eq. (\ref{equation_t_SNE_deriv_c1_y}) gives us:
\begin{align*}
\frac{\partial c_3}{\partial \b{y}_i} = 4 \sum_j (p_{ij} - q_{ij}) (1 + r_{ij})^{-1} (\b{y}_i - \b{y}_j),
\end{align*}
which is the gradient mentioned in the proposition. Q.E.D.

Note that in \cite{maaten2008visualizing}, the proof uses $z_{ij}$ rather than $r_{ij} = z_{ij}^2$ in Eq. (\ref{equation_t_SNE_r}) and the rest of the proof. In our opinion, it is better to use $z_{ij}^2$ rather than $z_{ij}$ for the proof.
\end{proof}

The update of the embedded point $\b{y}_i$ is done by gradient descent whose every iteration is as Eq. (\ref{equation_SNE_gradient_descent}) where $c_1$ is replaced by $c_3$.
For t-SNE, there is no need to add jitter to the solution of initial iterations \cite{maaten2008visualizing} because it is more robust than SNE.
The $\alpha(t)$ is the momentum which can be updated according to Eq. (\ref{equation_alpha_momentum}).
The $\eta$ is the learning rate which can be a small positive constant (e.g., $\eta=0.1$) or can be updated according to \cite{jacobs1988increased} (in \cite{maaten2008visualizing}, the initial $\eta$ is $100$).


Note that in \cite{maaten2008visualizing}, the update of $\b{y}_i^{(t)}$ is $\Delta \b{y}_i^{(t)} := + \eta\, \frac{\partial c_3}{\partial \b{y}_i} + \alpha(t)\, \Delta \b{y}_i^{(t-1)}$ which we think is a typo in that paper because the positive direction of gradient is used in gradient ascent for maximizing and not minimizing the objective function. We also checked the implementation of t-SNE in Python scikit-learn library\footnote{The link is: \texttt{https://github.com/scikit-learn/}
\newline
\texttt{scikit-learn/blob/7b136e9/sklearn/manifold/}
\newline
\texttt{t\_sne.py}} and it was gradient descent and not gradient ascent.

\subsection{Early Exaggeration}

In t-SNE, it is better to multiply all $p_{ij}$'s by a constant (e.g., $4$) in the initial iterations:
\begin{align}
p_{ij} := p_{ij} \times 4,
\end{align}
which is called \textit{early exaggeration}. This heuristic helps the optimization focus on the large $p_{ij}$'s (close neighbors) more in the early iterations. 
This is because large $p_{ij}$'s are affected more by multiplying by $4$ than the small $p_{ij}$'s.
After the neighbours are embedded close to one another, we are free not to do this multiplication any more and let far-away points be embedded using the probabilities without multiplication.
Note that the early exaggeration is optional and not mandatory. 

\section{General Degrees of Freedom in t-SNE}\label{section_tSNE_general_degrees}

We can have general degrees of freedom for Student-t distribution in t-SNE \cite{van2009learning}.
As we saw in Eqs. (\ref{equation_t_SNE_t_distribution}) and (\ref{equation_t_SNE_q_generalDegrees}), we can have any degrees of freedom for $q_{ij}$ (note that $\alpha$ is a positive integer). We repeat Eq. (\ref{equation_t_SNE_q_generalDegrees}) here for more convenience:
\begin{align}\label{equation_q_general_degree}
q_{ij} = \frac{(1 + z_{ij}^2/\delta)^{-(\delta+1)/2}}{\sum_{k \neq l}(1 + z_{kl}^2/\delta)^{-(\delta+1)/2}}.
\end{align}
If $\delta \rightarrow \infty$, the Student-t distribution formulated in Eq. (\ref{equation_t_SNE_t_distribution}) tends to Gaussian distribution used in SNE \cite{hinton2003stochastic}. 
SNE and t-SNE use degrees $\delta \rightarrow \infty$ and $\delta = 1$ in Eq. (\ref{equation_q_general_degree}), respectively. Note that the kernel $q_{ij}$ in the low-dimensional space has no need to be a probability distribution necessarily, but it is enough for it to be a decaying function. It has been shown in \cite{kobak2019heavy} the degree $\delta < 1$ works properly well for embedding.

There are three ways to determine $\delta$ \cite{van2009learning}: 

\begin{enumerate}

\item We can set $\delta$ to be fixed. For example, $\delta=1$ is used in the original t-SNE \cite{maaten2008visualizing} which uses the Cauchy distribution in Eq. (\ref{equation_t_SNE_q}). 

\item The problem of the first approach is not considering the relation of the crowding problem with the dimensionality of the embedded space. Recall the crowding problem discussed in Section \ref{section_crowding_problem}. 
On the one hand, as Eq. (\ref{equation_t_SNE_t_distribution}) shows, the degree of freedom is in the power so the tail thickness of Student-t distribution decreases exponentially with $\delta$. 
On the other hand, the volume of a hyper-sphere grows exponentially with the dimension; for example, in two and three dimensions, the volume is $\pi r^2$ and $(4/3) \pi r^3$, respectively, where $r$ is the radius. 
The crowding volume in the embedded space to store the embedded data points is $\propto \pi r^h$ and grows exponentially with $h$. Therefore, the relation of $\delta$ and $h$ (dimensionality of embedded space) is linear, i.e., $h \propto \delta$. In order to be consistent with the original t-SNE \cite{maaten2008visualizing}, we take $\delta = h - 1$ which gives $\delta=1$ for $h=2$ \cite{van2009learning}.

\item The problem of the second approach is that $\delta$ might not ``only'' depend on $h$. In this approach, we find the best $\alpha$ which minimizes the cost $c_3$, i.e. Eq. (\ref{equation_t_SNE_cost}), \cite{van2009learning} where $p_{ij}$ is obtained using Eqs. (\ref{equation_t_SNE_p_conditional}) and (\ref{equation_t_SNE_p}) and $q_{ij}$ is Eq. (\ref{equation_t_SNE_q_generalDegrees}).
We use gradient descent \cite{boyd2004convex} for optimization of both $\delta$ and $\{\b{y}_i\}_{i=1}^n$, where the gradients are as mentioned and proved before. 
The parametric t-SNE \cite{van2009learning} has used restricted Boltzmann machine \cite{hinton2006reducing,hinton2012practical,ghojogh2021restricted} to learn the optimal $\delta$ and $\{\b{y}_i\}_{i=1}^n$ by a neural network.  
One can use an alternating optimization approach \cite{jain2017non} to solve for both $\delta$ and $\{\b{y}_i\}_{i=1}^n$ simultaneously. In this approach, $\{\b{y}_i\}_{i=1}^n$ are updated with gradient descent using Eq. (\ref{equation_general_tSNE_gradient}); then, the degree $\delta$ is updated with gradient descent using Eq. (\ref{equation_general_tSNE_gradient_degrees}), and this procedure is repeated until convergence. 

Note that the degree $\delta$ is an integer greater than or equal to one. However, the gradient in Eq. (\ref{equation_general_tSNE_gradient_degrees}) is a float number. For updating the degree using gradient descent in the alternating optimization approach, one can update the degree using the sign of gradient, i.e.:
\begin{align}
\delta := \delta - \text{sign}(\frac{\partial c_3}{\partial \delta}),
\end{align}
because the direction of updating is opposite to the gradient direction.

\end{enumerate}

For convenience, we list $p_{ij}$, $q_{ij}$, and $c_3$ here again:
\begin{align}
& \mathbb{R} \ni p_{j|i} := \frac{\exp(-d_{ij}^2)}{\sum_{k \neq i}\exp(-d_{ik}^2)}, \\
& \mathbb{R} \ni p_{ij} := \frac{p_{i|j} + p_{j|i}}{2n}, \label{equation_general_t_SNE_p} \\
& \mathbb{R} \ni q_{ij} = \frac{(1 + z_{ij}^2/\delta)^{-(\delta+1)/2}}{\sum_{k \neq l}(1 + z_{kl}^2/\delta)^{-(\delta+1)/2}}, \label{equation_general_t_SNE_q} \\
& \mathbb{R} \ni c_3 := \sum_i \text{KL}(P_i||Q_i) = \sum_i \sum_{j \neq i} p_{ij} \log (\frac{p_{ij}}{q_{ij}}). \label{equation_general_t_SNE_cost}
\end{align}

\begin{proposition}
The gradient of $c_3$ with respect to $\delta$ is:
\begin{align}\label{equation_general_tSNE_gradient_degrees}
\frac{\partial c_3}{\partial \delta} = \sum_{i \neq j} \Big( \frac{-(1+\delta) z_{ij}^2}{2 \delta^2 (1 + \frac{z_{ij}^2}{\delta})} + \frac{1}{2} \log(1 + \frac{z_{ij}^2}{\delta}) \Big) (p_{ij} - q_{ij}),
\end{align}
where $p_{ij}$ and $q_{ij}$ are the Eqs. (\ref{equation_general_t_SNE_p}) and (\ref{equation_general_t_SNE_q}), respectively, and $z_{ij}^2 := ||\b{y}_i - \b{y}_j||_2^2$.
\end{proposition}

No matter which of the three ways of determining $\delta$ is used, we need to optimize the cost function $c_3$ (Eq. (\ref{equation_general_t_SNE_cost})) using gradient descent.

\begin{proposition}
The gradient of $c_3$ with respect to $\b{y}_i$ is:
\begin{equation}\label{equation_general_tSNE_gradient}
\begin{aligned}
&\frac{\partial c_3}{\partial \b{y}_i} = \frac{2 \delta + 2}{\delta} \times \\
&\sum_j (p_{ij} - q_{ij}) (1 + \frac{||\b{y}_i - \b{y}_j||_2^2}{\delta})^{-1}(\b{y}_i - \b{y}_j), 
\end{aligned}
\end{equation}
where $p_{ij}$ and $q_{ij}$ are the Eqs. (\ref{equation_general_t_SNE_p}) and (\ref{equation_general_t_SNE_q}), respectively.
\end{proposition}

\begin{proof}
Let:
\begin{align}\label{equation_general_t_SNE_r}
\mathbb{R} \ni r_{ij} := z_{ij}^2 = ||\b{y}_i - \b{y}_j||_2^2.
\end{align}
By changing $\b{y}_i$, we only have change impact in $z_{ij}$ and $z_{ji}$ for all $j$'s.
Considering Eq. (\ref{equation_general_t_SNE_r}) and according to chain rule, we have:
\begin{align*}
\mathbb{R}^p \ni \frac{\partial c_3}{\partial \b{y}_i} = \sum_j \big(\frac{\partial c_3}{\partial r_{ij}} \frac{\partial r_{ij}}{\partial \b{y}_i} + \frac{\partial c_3}{\partial r_{ji}} \frac{\partial r_{ji}}{\partial \b{y}_i}\big).
\end{align*}
According to Eq. (\ref{equation_general_t_SNE_r}), we have:
\begin{align*}
& r_{ij} = ||\b{y}_i - \b{y}_j||_2^2 \implies \frac{\partial r_{ij}}{\partial \b{y}_i} = 2(\b{y}_i - \b{y}_j), \\
& r_{ji} = ||\b{y}_j - \b{y}_i||_2^2 = ||\b{y}_i - \b{y}_j||_2^2 \implies \frac{\partial r_{ji}}{\partial \b{y}_i} = 2(\b{y}_i - \b{y}_j).
\end{align*}
Therefore:
\begin{align}\label{equation_general_t_SNE_deriv_c1_y}
\therefore ~~~~ \frac{\partial c_3}{\partial \b{y}_i} = 2 \sum_j \big(\frac{\partial c_3}{\partial r_{ij}} + \frac{\partial c_3}{\partial r_{ji}}\big)(\b{y}_i - \b{y}_j).
\end{align}
The cost function can be re-written as:
\begin{align*}
c_3 &= \sum_{k} \sum_{l\neq k} p_{kl} \log (\frac{p_{kl}}{q_{kl}}) = \sum_{k \neq l} p_{kl} \log (\frac{p_{kl}}{q_{kl}}) \\
&= \sum_{k \neq l} \big(p_{kl} \log (p_{kl}) - p_{kl} \log (q_{kl}) \big),
\end{align*}
whose first term is a constant with respect to $q_{kl}$ and thus to $r_{kl}$. We have:
\begin{align*}
\mathbb{R} \ni \frac{\partial c_3}{\partial r_{ij}} = - \sum_{k \neq l} p_{kl} \frac{\partial (\log (q_{kl}))}{\partial r_{ij}}.
\end{align*}
According to Eq. (\ref{equation_general_t_SNE_q}), the $q_{kl}$ is:
\begin{align*}
q_{kl} = \frac{(1 + r_{kl}/\delta)^{-(\delta+1)/2}}{\sum_{m \neq f}(1 + r_{mf}/\delta)^{-(\delta+1)/2}}, 
\end{align*}
We take the denominator of $q_{kl}$ as:
\begin{align}\label{equation_t_SNE_generalDegree_beta}
\beta := \sum_{m \neq f}(1 + r_{mf}/\delta)^{-(\delta+1)/2}.
\end{align}
We have $\log (q_{kl}) = \log (q_{kl}) + \log \beta - \log \beta = \log (q_{kl} \beta) - \log \beta$. Therefore:
\begin{align*}
\therefore ~~~ \frac{\partial c_3}{\partial r_{ij}} &= - \sum_{k \neq l} p_{kl} \frac{\partial \big(\log (q_{kl} \beta) - \log \beta\big)}{\partial r_{ij}} \\
&= - \sum_{k \neq l} p_{kl} \bigg[\frac{\partial \big(\log (q_{kl} \beta)\big)}{\partial r_{ij}} - \frac{\partial \big(\log \beta\big)}{\partial r_{ij}}\bigg] \\
&= - \sum_{k \neq l} p_{kl} \bigg[\frac{1}{q_{kl} \beta}\frac{\partial \big( q_{kl} \beta\big)}{\partial r_{ij}} - \frac{1}{\beta}\frac{\partial \beta}{\partial r_{ij}}\bigg].
\end{align*}
The $q_{kl} \beta$ is:
\begin{align*}
q_{kl} \beta = &\, \frac{(1 + r_{kl}/\delta)^{-(\delta+1)/2}}{\sum_{m \neq f}(1 + r_{mf}/\delta)^{-(\delta+1)/2}} \times \\
& \sum_{m \neq f}(1 + r_{mf}/\delta)^{-(\delta+1)/2} = (1 + r_{kl}/\delta)^{-(\delta+1)/2}.
\end{align*}
The $\partial \big( (1 + r_{kl}/\delta)^{-(\delta+1)/2} \big)/\partial r_{ij}$ is non-zero for only $k=i$ and $l=j$; therefore:
\begin{align*}
&\frac{\partial \big( q_{kl} \beta\big)}{\partial r_{ij}} = \frac{\partial \big( (1 + r_{kl}/\delta)^{-(\delta+1)/2} \big)}{\partial r_{ij}} \\
&= -\frac{\delta+1}{2\delta} (1 + \frac{r_{ij}}{\delta})^{-\frac{\delta+3}{2}}, \\
&\frac{\partial \beta}{\partial r_{ij}} = \frac{\partial \sum_{m \neq f}(1 + r_{mf}/\delta)^{-(\delta+1)/2}}{\partial r_{ij}} \\
&= \frac{\partial \big( (1 + r_{kl}/\delta)^{-(\delta+1)/2} \big)}{\partial r_{ij}} = -\frac{\delta+1}{2\delta} (1 + \frac{r_{ij}}{\delta})^{-\frac{\delta+3}{2}}.
\end{align*}
Therefore:
\begin{alignat*}{2}
\frac{\partial c_3}{\partial r_{ij}} &= &&- \bigg( p_{ij} \Big[\frac{-1}{q_{ij} \beta} \frac{\delta+1}{2\delta} (1 + \frac{r_{ij}}{\delta})^{-\frac{\delta+3}{2}} \Big] + 0 \\
& &&+ \dots + 0 \bigg) - \sum_{k \neq l} p_{k l} \Big[\frac{1}{\beta} \frac{\delta+1}{2\delta} (1 + \frac{r_{ij}}{\delta})^{-\frac{\delta+3}{2}} \Big].
\end{alignat*}
We have $\sum_{k \neq l} p_{kl} = 1$ because summation of all possible probabilities is one. Thus:
\begin{align}
\frac{\partial c_3}{\partial r_{ij}} &= -  p_{ij} \Big[\frac{-1}{q_{ij} \beta} \frac{\delta+1}{2\delta} (1 + \frac{r_{ij}}{\delta})^{-\frac{\delta+3}{2}} \Big] \nonumber \\
&- \Big[\frac{1}{\beta} \frac{\delta+1}{2\delta} (1 + \frac{r_{ij}}{\delta})^{-\frac{\delta+3}{2}} \Big] \nonumber \\
&= \frac{\delta+1}{2\delta} (1 + \frac{r_{ij}}{\delta})^{-\frac{\delta+3}{2}} \frac{1}{\beta} (\frac{p_{ij}}{q_{ij}} - 1) \nonumber \\
&= \frac{\delta+1}{2\delta} \underbrace{\frac{(1 + \frac{r_{ij}}{\delta})^{-\frac{\delta+1}{2}}}{\beta}}_{=q_{ij}}  (1 + \frac{r_{ij}}{\delta})^{-1} (\frac{p_{ij}}{q_{ij}} - 1) \label{equation_general_t_SNE_proofDerivation} \\
&= \frac{\delta+1}{2\delta} (1 + \frac{r_{ij}}{\delta})^{-1} (p_{ij} - q_{ij}). \nonumber
\end{align}
Similarly, we have:
\begin{align*}
\frac{\partial c_3}{\partial r_{ji}} &= \frac{\delta+1}{2\delta} (1 + \frac{r_{ji}}{\delta})^{-1} (p_{ji} - q_{ji}) \\
&\overset{(a)}{=} \frac{\delta+1}{2\delta} (1 + \frac{r_{ij}}{\delta})^{-1} (p_{ij} - q_{ij}),
\end{align*}
where $(a)$ is because in t-SNE with general degrees of freedom, the $p_{ij}$, $q_{ij}$, and $r_{ij}$ are symmetric for $i$ and $j$ according to Eqs. (\ref{equation_general_t_SNE_p}), (\ref{equation_general_t_SNE_q}), and (\ref{equation_general_t_SNE_r}).

Substituting the obtained derivatives in Eq. (\ref{equation_general_t_SNE_deriv_c1_y}) gives us:
\begin{align*}
\frac{\partial c_3}{\partial \b{y}_i} = 2 \sum_j \frac{\delta+1}{\delta} (1 + \frac{r_{ij}}{\delta})^{-1} (p_{ij} - q_{ij}) (\b{y}_i - \b{y}_j),
\end{align*}
which is the gradient mentioned in the proposition. Q.E.D.

Note that in \cite{van2009learning}, the gradient is mentioned to be:
\begin{align}
\frac{\partial c_3}{\partial \b{y}_i} = 2 \sum_j \frac{\delta+1}{\delta} (1 + \frac{r_{ij}}{\delta})^{-\frac{\delta+1}{2}} (p_{ij} - q_{ij}) (\b{y}_i - \b{y}_j),
\end{align}
which we think is wrong. We conjecture that, in the paper \cite{van2009learning}, there might have been a small mistake in derivation of Eq. (\ref{equation_general_t_SNE_proofDerivation}) where possibly $(1 + \frac{r_{ij}}{\delta})^{-1} / \beta$ has been taken rather than $(1 + \frac{r_{ij}}{\delta})^{-\frac{\delta+1}{2}} / \beta$ to be $q_{ij}$ because of the habit of the original t-SNE \cite{maaten2008visualizing}.
\end{proof}

Comparing Eqs. (\ref{equation_tSNE_gradient}) and (\ref{equation_general_tSNE_gradient}) shows that the original t-SNE \cite{maaten2008visualizing} is a special case with $\delta=1$.

\section{Out-of-sample Embedding}\label{section_outOfSample}


Recall that we have $n$ high-dimensional data points $\{\b{x}_i\}_{i=1}^n$ and we want to embed them into the lower dimensional data $\{\b{y}_i\}_{i=1}^n$ where $\b{x}_i \in \mathbb{R}^d$ and $\b{y}_i \in \mathbb{R}^p$. 
Assume we have $n_t$ out-of-sample data points $\{\b{x}_i^{(t)}\}_{i=1}^{n_t}$ and we want to embed them into the lower dimensional data $\{\b{y}_i^{t}\}_{i=1}^{n_t}$ where $\b{x}_i^{(t)} \in \mathbb{R}^d$ and $\b{y}_i^{(t)} \in \mathbb{R}^p$. 
There are several different methods for out-of-sample extension of SNE and t-SNE methods. One approach, which we do not cover in this manuscript, is based on optimization \cite{bunte2012general}. Another method is based on \textit{kernel mapping} \cite{gisbrecht2012out,gisbrecht2015parametric} which we explain in the following.



We define a map which maps any data point as $\b{x} \mapsto \b{y}(\b{x})$, where:
\begin{align}\label{equation_kernel_tSNE_map}
\mathbb{R}^p \ni \b{y}(\b{x}) := \sum_{j=1}^n \b{\alpha}_j\, \frac{k(\b{x}, \b{x}_j)}{\sum_{\ell=1}^n k(\b{x}, \b{x}_{\ell})},
\end{align}
and $\b{\alpha}_j \in \mathbb{R}^p$, and $\b{x}_j$ and $\b{x}_{\ell}$ denote the $j$-th and $\ell$-th training data points, respectively.
The $k(\b{x}, \b{x}_j)$ is a kernel such as the Gaussian kernel:
\begin{align}
k(\b{x}, \b{x}_j) = \exp(\frac{-||\b{x} - \b{x}_j||_2^2}{2\, \sigma_j^2}),
\end{align}
where $\sigma_j$ is calculated as \cite{gisbrecht2015parametric}:
\begin{align}
\sigma_j := \gamma \times \min_{i}(||\b{x}_j - \b{x}_i||_2),
\end{align}
where $\gamma$ is a small positive number.

Assume we have already embedded the training data points using SNE or t-SNE; therefore, the set $\{\b{y}_i\}_{i=1}^n$ is available.
If we map the training data points, we want to minimize the following least-squares cost function in order to get $\b{y}(\b{x}_i)$ close to $\b{y}_i$ for the $i$-th training point:
\begin{equation}
\begin{aligned}
& \underset{\b{\alpha}_j\text{'s}}{\text{minimize}}
& & \sum_{i=1}^n ||\b{y}_i - \b{y}(\b{x}_i)||_2^2,
\end{aligned}
\end{equation}
where the summation is over the training data points.
We can write this cost function in matrix form as:
\begin{equation}\label{equation_kernel_tSNE_leastSquares}
\begin{aligned}
& \underset{\b{A}}{\text{minimize}}
& & ||\b{Y} - \b{K}\b{A}||_F^2,
\end{aligned}
\end{equation}
where $\mathbb{R}^{n \times p} \ni \b{Y} := [\b{y}_1, \dots, \b{y}_n]^\top$ and $\mathbb{R}^{n \times p} \ni \b{A} := [\b{\alpha}_1, \dots, \b{\alpha}_n]^\top$. 
The $\b{K} \in \mathbb{R}^{n \times n}$ is the kernel matrix whose $(i,j)$-th element is:
\begin{align}
\b{K}(i,j) := \frac{k(\b{x}_i, \b{x}_j)}{\sum_{\ell=1}^n k(\b{x}_i, \b{x}_{\ell})}.
\end{align}
The Eq. (\ref{equation_kernel_tSNE_leastSquares}) is always non-negative; thus, its smallest value is zero.
Therefore, the solution to this equation is:
\begin{align}
\b{Y} - \b{K}\b{A} = \b{0} &\implies \b{Y} = \b{K}\b{A} \nonumber \\
&\overset{(a)}{\implies} \b{A} = \b{K}^{\dagger}\, \b{Y}, \label{equation_kernel_tSNE_A_matrix}
\end{align}
where $\b{K}^{\dagger}$ is the pseudo-inverse of $\b{K}$:
\begin{align}
\b{K}^{\dagger} = (\b{K}^\top \b{K})^{-1} \b{K}^\top,
\end{align}
and $(a)$ is because $\b{K}^{\dagger}\,\b{K} = \b{I}$.

Finally, the mapping of Eq. (\ref{equation_kernel_tSNE_map}) for the $n_t$ out-of-sample data points is:
\begin{align}\label{equation_kernel_tSNE_outOfSample_Y}
\b{Y}^{(t)} = \b{K}^{(t)}\,\b{A}, 
\end{align}
where $\mathbb{R}^{n_t \times p} \ni \b{Y}^{(t)} := [\b{y}_1^{(t)}, \dots, \b{y}_{n_t}^{(t)}]^\top$ and the $(i,j)$-th element of the out-of-sample kernel matrix $\b{K}^{(t)} \in \mathbb{R}^{n_t \times n}$ is:
\begin{align}
\b{K}^{(t)}(i,j) := \frac{k(\b{x}_i^{(t)}, \b{x}_j)}{\sum_{\ell=1}^n k(\b{x}_i^{(t)}, \b{x}_{\ell})},
\end{align}
where $\b{x}_i^{(t)}$ is the $i$-th out-of-sample data point, and $\b{x}_j$ and $\b{x}_{\ell}$ are the $j$-th and $\ell$-th training data points, respectively.

In Eq. (\ref{equation_kernel_tSNE_A_matrix}), if, $\b{Y}$ is the embedding of training data using SNE/t-SNE, then the out-of-sample embedding of SNE/t-SNE are obtained. 
As mentioned in \cite{gisbrecht2015parametric}, this method can also be used for out-of-sample extension in Isomap \cite{tenenbaum2000global,ghojogh2020multidimensional}, Locally Linear Embedding (LLE) \cite{roweis2000nonlinear}, and Maximum Variance Embedding (MVU) \cite{weinberger2006unsupervised}. The only difference is in obtaining the embedded training points $\b{Y}$ using different non-parametric dimensionality reduction methods \cite{gisbrecht2012out,gisbrecht2015parametric}.

\section{Accelerating SNE and t-SNE}\label{section_acceleration}

The SNE and t-SNE methods are very slow because of numerical iterative optimization. 
Different methods have been proposed for accelerating these methods \cite{linderman2017efficient}. 
Some of these methods are the \textit{tree-based} algorithms \cite{van2014accelerating,robinson2020tree}. 
This type of t-SNE is also referred to as \textit{Barnes-Hut t-SNE} \cite{van2013barnes,van2014accelerating}.
We do not cover the tree-based algorithms \cite{van2014accelerating} in this paper for the sake of brevity.
Some other methods exist for accelerating SNE and t-SNE which are based on landmarks. In these methods, we randomly sample from the dataset in order to have a subset of data. The sampled data points are called \textit{landmarks}. In the following, we mention three methods for accelerating t-SNE and/or SNE which use landmarks.
In the following, we review the methods based on landmarks.

\subsection{Acceleration Using Out-of-sample Embedding}


One way to speed up SNE and t-SNE is the kernel mapping \cite{gisbrecht2015parametric} introduced in Section \ref{section_outOfSample}. We consider the landmarks as the training data points and train SNE or t-SNE with them. Thereafter, we treat the non-landmark data points as out-of-sample points. We use kernel SNE or kernel t-SNE (or Fisher kernel t-SNE for supervised cases) in order to embed the out-of-sample data points.


Another method is to again to consider the landmarks as training points and embed the training points using SNE or t-SNE. Then, the non-landmarks, which are out-of-sample points, are embedded using optimization \cite{bunte2012general} as was mentioned in Section \ref{section_outOfSample}.

\subsection{Acceleration Using Random Walk}

Another way of accelerating t-SNE is random walk \cite{maaten2008visualizing}. 
First, a $k$-Nearest Neighbor ($k$NN) graph is constructed using all points including landmarks and non-landmarks. This method has an acceptable robustness to the choice of $k$; for example, $k=20$ can be used \cite{maaten2008visualizing}.
Also, note that calculation of $k$NN is time-consuming for large dataset; however, it is not a big deal as it is done only once.
Then, multiple random walks are performed in this $k$NN graph \cite{spitzer2013principles}.
For every random walk, we start from a random landmark and randomly select the edges and go further randomly until we reach another landmark and then we terminate for that random walk.
After performing all the random walks, the fraction of random walks which pass through the point $\b{x}_i$ (either landmark or non-landmark) and then reach the point $\b{x}_j$ (either landmark or non-landmark) is a good approximation for $p_{j|i}$. In t-SNE, we use this approximation in place of Eq. (\ref{equation_t_SNE_p_conditional}), which is then used in Eq. (\ref{equation_t_SNE_p}).
The rest of t-SNE is similar to the original t-SNE. Therefore, for $p_{ij}$ in Eq. (\ref{equation_tSNE_gradient}), we use the approximation rather than Eq. (\ref{equation_t_SNE_p}) and this makes the t-SNE much faster.

\section{Recent Improvements of t-SNE}\label{section_improvement_of_tSNE}

Here, we just list some of the recent improvements of t-SNE and do not explain them in detail for brevity. 
Recall that the variance $\sigma_i^2$ is determined for every point $\b{x}_i$ using binary search. This cancels the local density information for points because a point in denser regions will have a smaller $\sigma_i^2$. Dense t-SNE \cite{narayan2021assessing} resolves this problem by a density radius to include the density information. LargeVis \cite{tang2016visualizing} and UMAP \cite{mcinnes2018umap,ghojogh2021uniform} are closely related methods to t-SNE. Parametric t-SNE \cite{van2009learning} and parametric kernel t-SNE \cite{gisbrecht2015parametric} implement t-SNE formulation in a neural network structure. 
The optimization of t-SNE can be seen as optimizing attractive and repulsive forces between points. Some discussions on the attractive and repulsive forces in t-SNE can be found in \cite{linderman2017efficient,sainburg2020parametric}. Many algorithms such as t-SNE, which are based on these forces, can be unified as a family of neighborhood embedding methods \cite{bohm2020unifying,bohm2020dimensionality}.
Finally, note that a combination of variational autoencoder \cite{kingma2014auto,ghojogh2021factor} and SNE exists \cite{graving2020vae}.

\section{Conclusion}\label{section_conclusion}

This paper was a tutorial and survey paper on SNE and its variants. These methods have a probabilistic approach where the probabilities of neighborhood in the input space are tried to be preserved in the embedding space. We explained SNE, symmetric SNE, t-SNE (or Cauchy-SNE), and t-SNE with general degrees of freedom. We also covered out-of-sample extension and their acceleration methods. Finally, some simulations were provided for visualization of embeddings. 

Some newer variants of SNE and t-SNE were not covered in this manuscript and we refer the reader to those papers for more information. An example is \textit{Fisher kernel t-SNE} \cite{gisbrecht2015parametric} for \textit{supervised} embedding using t-SNE. This method uses \textit{$T$-point approximation} of Riemannian distance \cite{peltonen2004improved} in the formulation of probability. 
There is also some other technique for heavy-tailed SNE such as \cite{yang2009heavy}. 

\bibliography{References}

\begin{thebibliography}{51}
\providecommand{\natexlab}[1]{#1}
\providecommand{\url}[1]{\texttt{#1}}
\expandafter\ifx\csname urlstyle\endcsname\relax
  \providecommand{\doi}[1]{doi: #1}\else
  \providecommand{\doi}{doi: \begingroup \urlstyle{rm}\Url}\fi

\bibitem[B{\"o}hm(2020)]{bohm2020dimensionality}
B{\"o}hm, Jan~Niklas.
\newblock \emph{Dimensionality Reduction with Neighborhood Embeddings}.
\newblock PhD thesis, University of T{\"u}bingen, 2020.

\bibitem[B{\"o}hm et~al.(2020)B{\"o}hm, Berens, and Kobak]{bohm2020unifying}
B{\"o}hm, Jan~Niklas, Berens, Philipp, and Kobak, Dmitry.
\newblock A unifying perspective on neighbor embeddings along the
  attraction-repulsion spectrum.
\newblock \emph{arXiv preprint arXiv:2007.08902}, 2020.

\bibitem[Boyd \& Vandenberghe(2004)Boyd and Vandenberghe]{boyd2004convex}
Boyd, Stephen and Vandenberghe, Lieven.
\newblock \emph{Convex optimization}.
\newblock Cambridge university press, 2004.

\bibitem[Bunte et~al.(2012)Bunte, Biehl, and Hammer]{bunte2012general}
Bunte, Kerstin, Biehl, Michael, and Hammer, Barbara.
\newblock A general framework for dimensionality-reducing data visualization
  mapping.
\newblock \emph{Neural Computation}, 24\penalty0 (3):\penalty0 771--804, 2012.

\bibitem[Ghojogh et~al.(2019)Ghojogh, Samad, Mashhadi, Kapoor, Ali, Karray, and
  Crowley]{ghojogh2019feature}
Ghojogh, Benyamin, Samad, Maria~N, Mashhadi, Sayema~Asif, Kapoor, Tania, Ali,
  Wahab, Karray, Fakhri, and Crowley, Mark.
\newblock Feature selection and feature extraction in pattern analysis: A
  literature review.
\newblock \emph{arXiv preprint arXiv:1905.02845}, 2019.

\bibitem[Ghojogh et~al.(2020)Ghojogh, Ghodsi, Karray, and
  Crowley]{ghojogh2020multidimensional}
Ghojogh, Benyamin, Ghodsi, Ali, Karray, Fakhri, and Crowley, Mark.
\newblock Multidimensional scaling, {Sammon} mapping, and {Isomap}: Tutorial
  and survey.
\newblock \emph{arXiv preprint arXiv:2009.08136}, 2020.

\bibitem[Ghojogh et~al.(2021{\natexlab{a}})Ghojogh, Ghodsi, Karray, and
  Crowley]{ghojogh2021factor}
Ghojogh, Benyamin, Ghodsi, Ali, Karray, Fakhri, and Crowley, Mark.
\newblock Factor analysis, probabilistic principal component analysis,
  variational inference, and variational autoencoder: Tutorial and survey.
\newblock \emph{arXiv preprint arXiv:2101.00734}, 2021{\natexlab{a}}.

\bibitem[Ghojogh et~al.(2021{\natexlab{b}})Ghojogh, Ghodsi, Karray, and
  Crowley]{ghojogh2021restricted}
Ghojogh, Benyamin, Ghodsi, Ali, Karray, Fakhri, and Crowley, Mark.
\newblock Restricted {Boltzmann} machine and deep belief network: Tutorial and
  survey.
\newblock \emph{arXiv preprint arXiv:2107.12521}, 2021{\natexlab{b}}.

\bibitem[Ghojogh et~al.(2021{\natexlab{c}})Ghojogh, Ghodsi, Karray, and
  Crowley]{ghojogh2021uniform}
Ghojogh, Benyamin, Ghodsi, Ali, Karray, Fakhri, and Crowley, Mark.
\newblock Uniform manifold approximation and projection ({UMAP}) and its
  variants: Tutorial and survey.
\newblock \emph{arXiv preprint arXiv:2109.02508}, 2021{\natexlab{c}}.

\bibitem[Gisbrecht et~al.(2012)Gisbrecht, Lueks, Mokbel, and
  Hammer]{gisbrecht2012out}
Gisbrecht, Andrej, Lueks, Wouter, Mokbel, Bassam, and Hammer, Barbara.
\newblock Out-of-sample kernel extensions for nonparametric dimensionality
  reduction.
\newblock In \emph{European Symposium on Artificial Neural Networks,
  Computational Intelligence and Machine Learning}, volume 2012, pp.\
  531--536, 2012.

\bibitem[Gisbrecht et~al.(2015)Gisbrecht, Schulz, and
  Hammer]{gisbrecht2015parametric}
Gisbrecht, Andrej, Schulz, Alexander, and Hammer, Barbara.
\newblock Parametric nonlinear dimensionality reduction using kernel t-sne.
\newblock \emph{Neurocomputing}, 147:\penalty0 71--82, 2015.

\bibitem[Globerson et~al.(2007)Globerson, Chechik, Pereira, and
  Tishby]{globerson2007euclidean}
Globerson, Amir, Chechik, Gal, Pereira, Fernando, and Tishby, Naftali.
\newblock Euclidean embedding of co-occurrence data.
\newblock \emph{Journal of Machine Learning Research}, 8\penalty0
  (Oct):\penalty0 2265--2295, 2007.

\bibitem[Goldberg \& Levy(2014)Goldberg and Levy]{goldberg2014word2vec}
Goldberg, Yoav and Levy, Omer.
\newblock word2vec explained: deriving mikolov et al.'s negative-sampling
  word-embedding method.
\newblock \emph{arXiv preprint arXiv:1402.3722}, 2014.

\bibitem[Goldberger et~al.(2005)Goldberger, Hinton, Roweis, and
  Salakhutdinov]{goldberger2005neighbourhood}
Goldberger, Jacob, Hinton, Geoffrey~E, Roweis, Sam~T, and Salakhutdinov,
  Russ~R.
\newblock Neighbourhood components analysis.
\newblock In \emph{Advances in neural information processing systems}, pp.\
  513--520, 2005.

\bibitem[Gosset~{(Student)}(1908)]{student1908probable}
Gosset~{(Student)}, William~Sealy.
\newblock The probable error of a mean.
\newblock \emph{Biometrika}, pp.\  1--25, 1908.

\bibitem[Graving \& Couzin(2020)Graving and Couzin]{graving2020vae}
Graving, Jacob~M and Couzin, Iain~D.
\newblock {VAE}-{SNE}: a deep generative model for simultaneous dimensionality
  reduction and clustering.
\newblock \emph{BioRxiv}, 2020.

\bibitem[Hinton(2012)]{hinton2012practical}
Hinton, Geoffrey~E.
\newblock A practical guide to training restricted boltzmann machines.
\newblock In \emph{Neural networks: Tricks of the trade}, pp.\  599--619.
  Springer, 2012.

\bibitem[Hinton \& Roweis(2003)Hinton and Roweis]{hinton2003stochastic}
Hinton, Geoffrey~E and Roweis, Sam~T.
\newblock Stochastic neighbor embedding.
\newblock In \emph{Advances in neural information processing systems}, pp.\
  857--864, 2003.

\bibitem[Hinton \& Salakhutdinov(2006)Hinton and
  Salakhutdinov]{hinton2006reducing}
Hinton, Geoffrey~E and Salakhutdinov, Ruslan~R.
\newblock Reducing the dimensionality of data with neural networks.
\newblock \emph{science}, 313\penalty0 (5786):\penalty0 504--507, 2006.

\bibitem[Im et~al.(2018)Im, Verma, and Branson]{im2018stochastic}
Im, Daniel~Jiwoong, Verma, Nakul, and Branson, Kristin.
\newblock Stochastic neighbor embedding under f-divergences.
\newblock \emph{arXiv preprint arXiv:1811.01247}, 2018.

\bibitem[Iwata et~al.(2005)Iwata, Saito, Ueda, Stromsten, Griffiths, and
  Tenenbaum]{iwata2005parametric}
Iwata, Tomoharu, Saito, Kazumi, Ueda, Naonori, Stromsten, Sean, Griffiths,
  Thomas~L, and Tenenbaum, Joshua~B.
\newblock Parametric embedding for class visualization.
\newblock In \emph{Advances in neural information processing systems}, pp.\
  617--624, 2005.

\bibitem[Jacobs(1988)]{jacobs1988increased}
Jacobs, Robert~A.
\newblock Increased rates of convergence through learning rate adaptation.
\newblock \emph{Neural networks}, 1\penalty0 (4):\penalty0 295--307, 1988.

\bibitem[Jain \& Kar(2017)Jain and Kar]{jain2017non}
Jain, Prateek and Kar, Purushottam.
\newblock Non-convex optimization for machine learning.
\newblock \emph{arXiv preprint arXiv:1712.07897}, 2017.

\bibitem[Kingma \& Welling(2014)Kingma and Welling]{kingma2014auto}
Kingma, Diederik~P and Welling, Max.
\newblock Auto-encoding variational {Bayes}.
\newblock In \emph{International Conference on Learning Representations}, 2014.

\bibitem[Kleinbaum et~al.(2002)Kleinbaum, Dietz, Gail, Klein, and
  Klein]{kleinbaum2002logistic}
Kleinbaum, David~G, Dietz, K, Gail, M, Klein, Mitchel, and Klein, Mitchell.
\newblock \emph{Logistic regression}.
\newblock Springer, 2002.

\bibitem[Kobak \& Berens(2019)Kobak and Berens]{kobak2019art}
Kobak, Dmitry and Berens, Philipp.
\newblock The art of using t-{SNE} for single-cell transcriptomics.
\newblock \emph{Nature communications}, 10\penalty0 (1):\penalty0 1--14, 2019.

\bibitem[Kobak et~al.(2019)Kobak, Linderman, Steinerberger, Kluger, and
  Berens]{kobak2019heavy}
Kobak, Dmitry, Linderman, George, Steinerberger, Stefan, Kluger, Yuval, and
  Berens, Philipp.
\newblock Heavy-tailed kernels reveal a finer cluster structure in t-{SNE}
  visualisations.
\newblock In \emph{Joint European Conference on Machine Learning and Knowledge
  Discovery in Databases}, pp.\  124--139. Springer, 2019.

\bibitem[Kullback(1997)]{kullback1997information}
Kullback, Solomon.
\newblock \emph{Information theory and statistics}.
\newblock Courier Corporation, 1997.

\bibitem[Linderman et~al.(2017)Linderman, Rachh, Hoskins, Steinerberger, and
  Kluger]{linderman2017efficient}
Linderman, George~C, Rachh, Manas, Hoskins, Jeremy~G, Steinerberger, Stefan,
  and Kluger, Yuval.
\newblock Efficient algorithms for t-distributed stochastic neighborhood
  embedding.
\newblock \emph{arXiv preprint arXiv:1712.09005}, 2017.

\bibitem[Liu et~al.(2020)Liu, Yang, Wang, and Hong]{liu2020deep}
Liu, Xueliang, Yang, Xun, Wang, Meng, and Hong, Richang.
\newblock Deep neighborhood component analysis for visual similarity modeling.
\newblock \emph{ACM Transactions on Intelligent Systems and Technology (TIST)},
  11\penalty0 (3):\penalty0 1--15, 2020.

\bibitem[McInnes et~al.(2018)McInnes, Healy, and Melville]{mcinnes2018umap}
McInnes, Leland, Healy, John, and Melville, James.
\newblock {UMAP}: Uniform manifold approximation and projection for dimension
  reduction.
\newblock \emph{arXiv preprint arXiv:1802.03426}, 2018.

\bibitem[Mikolov et~al.(2013{\natexlab{a}})Mikolov, Chen, Corrado, and
  Dean]{mikolov2013efficient}
Mikolov, Tomas, Chen, Kai, Corrado, Greg, and Dean, Jeffrey.
\newblock Efficient estimation of word representations in vector space.
\newblock \emph{arXiv preprint arXiv:1301.3781}, 2013{\natexlab{a}}.

\bibitem[Mikolov et~al.(2013{\natexlab{b}})Mikolov, Sutskever, Chen, Corrado,
  and Dean]{mikolov2013distributed}
Mikolov, Tomas, Sutskever, Ilya, Chen, Kai, Corrado, Greg~S, and Dean, Jeff.
\newblock Distributed representations of words and phrases and their
  compositionality.
\newblock In \emph{Advances in neural information processing systems}, pp.\
  3111--3119, 2013{\natexlab{b}}.

\bibitem[Movshovitz-Attias et~al.(2017)Movshovitz-Attias, Toshev, Leung, Ioffe,
  and Singh]{movshovitz2017no}
Movshovitz-Attias, Yair, Toshev, Alexander, Leung, Thomas~K, Ioffe, Sergey, and
  Singh, Saurabh.
\newblock No fuss distance metric learning using proxies.
\newblock In \emph{Proceedings of the IEEE International Conference on Computer
  Vision}, pp.\  360--368, 2017.

\bibitem[Narayan et~al.(2021)Narayan, Berger, and Cho]{narayan2021assessing}
Narayan, Ashwin, Berger, Bonnie, and Cho, Hyunghoon.
\newblock Assessing single-cell transcriptomic variability through
  density-preserving data visualization.
\newblock \emph{Nature Biotechnology}, 39\penalty0 (6):\penalty0 765--774,
  2021.

\bibitem[Peltonen et~al.(2004)Peltonen, Klami, and Kaski]{peltonen2004improved}
Peltonen, Jaakko, Klami, Arto, and Kaski, Samuel.
\newblock Improved learning of {Riemannian} metrics for exploratory analysis.
\newblock \emph{Neural Networks}, 17\penalty0 (8-9):\penalty0 1087--1100, 2004.

\bibitem[Qian(1999)]{qian1999momentum}
Qian, Ning.
\newblock On the momentum term in gradient descent learning algorithms.
\newblock \emph{Neural networks}, 12\penalty0 (1):\penalty0 145--151, 1999.

\bibitem[Robinson \& Pierce-Hoffman(2020)Robinson and
  Pierce-Hoffman]{robinson2020tree}
Robinson, Isaac and Pierce-Hoffman, Emma.
\newblock Tree-{SNE}: Hierarchical clustering and visualization using t-{SNE}.
\newblock \emph{arXiv preprint arXiv:2002.05687}, 2020.

\bibitem[Rong(2014)]{rong2014word2vec}
Rong, Xin.
\newblock word2vec parameter learning explained.
\newblock \emph{arXiv preprint arXiv:1411.2738}, 2014.

\bibitem[Roweis \& Saul(2000)Roweis and Saul]{roweis2000nonlinear}
Roweis, Sam~T and Saul, Lawrence~K.
\newblock Nonlinear dimensionality reduction by locally linear embedding.
\newblock \emph{Science}, 290\penalty0 (5500):\penalty0 2323--2326, 2000.

\bibitem[Sainburg et~al.(2020)Sainburg, McInnes, and
  Gentner]{sainburg2020parametric}
Sainburg, Tim, McInnes, Leland, and Gentner, Timothy~Q.
\newblock Parametric {UMAP}: learning embeddings with deep neural networks for
  representation and semi-supervised learning.
\newblock 2020.

\bibitem[Saul \& Roweis(2003)Saul and Roweis]{saul2003think}
Saul, Lawrence~K and Roweis, Sam~T.
\newblock Think globally, fit locally: unsupervised learning of low dimensional
  manifolds.
\newblock \emph{Journal of machine learning research}, 4\penalty0
  (Jun):\penalty0 119--155, 2003.

\bibitem[Spitzer(2013)]{spitzer2013principles}
Spitzer, Frank.
\newblock \emph{Principles of random walk}, volume~34.
\newblock Springer Science \& Business Media, 2013.

\bibitem[Tang et~al.(2016)Tang, Liu, Zhang, and Mei]{tang2016visualizing}
Tang, Jian, Liu, Jingzhou, Zhang, Ming, and Mei, Qiaozhu.
\newblock Visualizing large-scale and high-dimensional data.
\newblock In \emph{Proceedings of the 25th international conference on world
  wide web}, pp.\  287--297, 2016.

\bibitem[Tenenbaum et~al.(2000)Tenenbaum, De~Silva, and
  Langford]{tenenbaum2000global}
Tenenbaum, Joshua~B, De~Silva, Vin, and Langford, John~C.
\newblock A global geometric framework for nonlinear dimensionality reduction.
\newblock \emph{Science}, 290\penalty0 (5500):\penalty0 2319--2323, 2000.

\bibitem[van~der Maaten(2009)]{van2009learning}
van~der Maaten, Laurens.
\newblock Learning a parametric embedding by preserving local structure.
\newblock In \emph{Artificial Intelligence and Statistics}, pp.\  384--391,
  2009.

\bibitem[Van Der~Maaten(2013)]{van2013barnes}
Van Der~Maaten, Laurens.
\newblock Barnes-{Hut}-{SNE}.
\newblock \emph{arXiv preprint arXiv:1301.3342}, 2013.

\bibitem[van~der Maaten(2014)]{van2014accelerating}
van~der Maaten, Laurens.
\newblock Accelerating {t-SNE} using tree-based algorithms.
\newblock \emph{The Journal of Machine Learning Research}, 15\penalty0
  (1):\penalty0 3221--3245, 2014.

\bibitem[van~der Maaten \& Hinton(2008)van~der Maaten and
  Hinton]{maaten2008visualizing}
van~der Maaten, Laurens and Hinton, Geoffrey.
\newblock Visualizing data using t-{SNE}.
\newblock \emph{Journal of machine learning research}, 9\penalty0
  (Nov):\penalty0 2579--2605, 2008.

\bibitem[Weinberger \& Saul(2006)Weinberger and
  Saul]{weinberger2006unsupervised}
Weinberger, Kilian~Q and Saul, Lawrence~K.
\newblock Unsupervised learning of image manifolds by semidefinite programming.
\newblock \emph{International journal of computer vision}, 70\penalty0
  (1):\penalty0 77--90, 2006.

\bibitem[Yang et~al.(2009)Yang, King, Xu, and Oja]{yang2009heavy}
Yang, Zhirong, King, Irwin, Xu, Zenglin, and Oja, Erkki.
\newblock Heavy-tailed symmetric stochastic neighbor embedding.
\newblock In \emph{Advances in neural information processing systems}, pp.\
  2169--2177, 2009.

\end{thebibliography}
\bibliographystyle{icml2016}

\end{document}